%% file: main.tex
\documentclass[twoside]{article}

\usepackage[hidelinks]{hyperref}
\usepackage[accepted]{aistats2020}

\input{imports.tex}

\input{notation.tex}

\usepackage[round]{natbib}

\begin{document}

\twocolumn[

\aistatstitle{Bandit optimisation of functions in the Mat\'ern kernel RKHS}

\aistatsauthor{ David Janz \And David R. Burt \And  Javier Gonz\'alez }

\aistatsaddress{ University of Cambridge \And  University of Cambridge \And Amazon } 
]

\begin{abstract}
We consider the problem of optimising functions in the reproducing kernel Hilbert space (RKHS) of a Mat\'ern kernel with smoothness parameter $\nu$ over the domain $[0,1]^d$ under noisy bandit feedback. Our contribution, the $\pi$-GP-UCB algorithm, is the first practical approach with guaranteed sublinear regret for all $\nu>1$ and $d \geq 1$. Empirical validation suggests better performance and drastically improved computational scalability compared with its predecessor, Improved GP-UCB.
\end{abstract}

\section{Introduction}
We consider the black-box function optimisation problem using the stochastic bandit formalism of sequential decision making \citep{robbins1952some}. Under this model, we consider an agent that sequentially selects an action $x_t$ from an action set $\dom$ at each time step $t=1,\dotsc,T$ for $T \in \N$ and observes
\[ 
    y(x_t) = f(x_t) + \epsilon_t,
\]
where $y$ is the reward function, $f$ is the expected reward function and we assume $\epsilon_t$ is conditionally subGaussian given $x_t$. The agent's goal is to minimise regret, given by
\[ 
    R_T = \sum_{t=1}^T f(x^\star)-f(x_t),
\]
where $f(x^\star)$ is the reward associated with an optimal arm. Regret is closely linked to bounds on the convergence of black-box optimisation in the presence of noise: bounding $R_T$ by a quantity sublinear in $T$ implies $f(x^\star) - \max_{t \leq T} f(x_t)\leq R_T/T \to 0$ as $T \to \infty$, and so an algorithm achieving such a bound will converge to a subset of the global optima of $f$. 

A standard approach to the bandit problem is to construct a $1-\delta$ confidence upper bound for $f$, of the form
\[ \mu_{t-1}(x) + \beta_t(\delta) \sigma_{t-1}(x), \]
where $\mu_{t-1}$ and $\sigma_{t-1}$ are the mean and standard deviation predictors for $f$ given by a suitable estimator based on observations prior to time $t$, and $\beta_t(\delta)$ is a confidence width multiplier that ensures the expression is an upper bound on $f$ with probability at least $1-\delta$. Algorithms then select $x_t \in \dom$ that maximises this bound \citep{auer2002using,auer2002finite,auer2010ucb}. This upper confidence bound (UCB) strategy naturally balances exploration, sampling in regions where the uncertainty is large, with exploitation, focusing on regions where the mean is large, and leads to algorithms with minimax optimal regret bounds for the case where $\dom$ is finite \citep{audibert2009minimax}. 

From the perspective of black-box function optimisation, a particularly interesting bandit problem is the kernelised continuum-armed bandit \citep{srinivas2010gaussian}. Here,~$f$ is assumed to be in the closure of functions on $[0,1]^d$ expressible as a linear combination of a feature embedding parameterised by a kernel~$k$. The properties of the functions in the resulting space, referred to as the RKHS of $k$, are determined by the choice of the kernel. For example, the RKHS corresponding to the linear kernel contains linear functions, and in this case existing kernelised bandit algorithms recover bounds that match those of the relevant stochastic linear bandit algorithms \citep{chowdhury2017kernelized,abbasi2011improved}. For a squared exponential kernel, the corresponding RKHS contains only infinitely differentiable functions, and here the existing methods match known lower bounds up to polylogarithmic factors \citep{srinivas2010gaussian,scarlett17lower}. 

In this work, we focus on the RKHS associated with a Mat\'ern kernel, parameterised by a smoothness parameter $\nu$ \citep{stein2012interpolation}. For a given $\nu$, the Mat\'ern RKHS contains all $\lfloor \nu \rfloor$-times continuously differentiable functions, and therefore for any $\nu< \infty$, contains as a strict subset the RKHS of both the linear and the squared exponential kernels. The Mat\'ern RKHS is of particular practical significance, since it offers a more suitable set of assumptions for the modelling and optimisation of physical quantities \citep{stein2012interpolation,rasmussen2006gaussian}. However, the theoretical guarantees offered for this class of functions by existing kernelised algorithms, such as KernelUCB \citep{valko2013finite}, GP-UCB \citep{srinivas2010gaussian} and Improved GP-UCB \citep{chowdhury2017kernelized}, are limited. Specifically, these guarantee that the regret after $T$ steps, $R_T$, is bounded with high probability as\footnote{$f(t) = \porder{t^a} \iff \forall \epsilon > 0, \ f(t) = \mcO(t^{a+\epsilon})$.}
\[ 
    R_T = \widetilde{O}\left(T^{\min\left\{1, \frac{3d^2 + 2\nu}{2d^2+4\nu}\right\}}\right),
\]
leaving a large gap to the $\Omega(T^{\frac{d + \nu}{d + 2\nu}})$ algorithm-agnostic lower bound for this problem \citep{scarlett17lower}. Since this bound is linear for $2\nu \leq d^2$, existing practical kernel-based algorithms are not guaranteed to converge for functions with fewer than $\lfloor d^2/2\rfloor$ derivatives.

Our main contribution is an algorithm that successfully tackles the Mat\'ern RKHS bandit problem for all $\nu > 1$ and all $d \geq 1$. The algorithm, Partitioned Improved GP-UCB ($\pi$-GP-UCB), offers a high probability regret bound of order
\begin{equation}\label{eq:final-regret}
    R_T = \widetilde{O}\left(T^{\frac{d(2d+3)+2\nu}{d(2d+4) + 4\nu}}\right),
\end{equation}
and therefore guarantees convergence for once differentiable functions of any finite dimensionality. Our contribution is not limited to theory: $\pi$-GP-UCB shows strong empirical performance on 3-dimensional functions in the Mat\'ern $\nu=3/2$ RKHS (in contrast, Improved GP-UCB barely outperforms uniformly sampling of arms). Moreover, our experiments show that despite using Gaussian process regression to construct confidence intervals, $\pi$-GP-UCB achieves an empirically near-linear runtime. 

Our analysis also results in tighter bounds on the effective dimension associated with the Mat\'ern kernel RKHS, an important quantity in the context of kernelised bandit problems which immediately improves bounds for a range of existing algorithms.

\section{Background}\label{sec:background}
Our work builds on Improved GP-UCB (IGP-UCB), and like IGP-UCB, it uses Gaussian process regression to construct confidence intervals. We briefly outline these topics and introduce the required notation.
 
\paragraph{Gaussian process regression}
We use $\dat{\dom}{t}= \{(x_i, y_i) \colon i \leq t\}$ to denote a set of $t$ observations with $x_i \in \dom$, and
\[
\dat{A}{t}  = \{(x^A_i, y^A_i) \in \dat{\dom}{t} \colon x^A_i \in A\}
\] 
for $A \subset \dom$ to denote the subset of these located in $A$. We denote the sequence of input locations within $A$ by $X_t^A$ and by $(y^A_1, \dotsc, y^A_N)$ the associated sequence of observations.

For $\dat{A}{t}$ of cardinality $N$, $x \in A$, a kernel $k$, assumed to be normalised such that $k(x, x) = 1$, we define 
\[
k^A_t(x)=[k(x^A_1, x), \dotsc, k(x^A_N, x)]^T,
\]
as well as
\[ 
\kmat{A}{t} = [k(x, x^\prime)]_{x, x^\prime \in X^A_t} \ \ \text{and} \ \  y^A_{1:t} = [y^A_1, \dotsc, y^A_N]^T.
\]

For a regularisation parameter $\reg>0$, we define the Gaussian process regressor on $A \subset \dom$ by a mean,
\[
\mean{A}{t}{x} = k^A_t(x)^T(\kmat{A}{t} + \reg I)^{-1} y^A_{1:t}, 
\]
and an associated predictive standard deviation,
\[
\std{A}{t}{x} = \sqrt{k(x, x) - k^A_t(x)^T(\kmat{A}{t} + \reg I)^{-1} k^A_t(x)},
\]
for each $x \in A$. Note that $\std{A}{t}{x}$ is monotone decreasing in both $t$ and $A$, meaning $\std{A}{t}{x}\leq\std{A'}{t}{x}$ if $A' \subset A$ and $\std{A}{t}{x}\leq\std{A}{t+1}{x}$ \citep{vivarelli1998studies}.

\paragraph{Effective dimension}
While Gaussian process regressors are often defined through an infinite dimensional feature embedding, due to finite data and regularisation, the number of features that have a noticeable impact on the regression model can be small.  The effective dimension of the Gaussian process regressor on the set $A \subset \dom$, defined as
\[
\deffe{A} =\text{Tr}\big(\kmat{A}{T}(\kmat{A}{T}+\alpha I)^{-1}\big)= \alpha^{-1}\sum_{t=1}^{T} \var{A}{T}{x_t},
\]
provides an estimate of the number of relevant features used in the regression problem \citep{zhang2005learning,valko2013finite}, and frequently appears in the bounds on kernelised bandit algorithms. It is closely related to information gain, defined 
\[
\info{A}{t} = \frac{1}{2}\log|I + \reg^{-1}\kmat{A}{t}|
\] 
for $A \subset \dom$, with the two of the same order up to polylogarithmic factors \citep[proposition 5]{calandriello2019gaussian}. We shall use this relationship throughout.

\paragraph{Mat\'ern kernel family} The Mat\'ern family consists of kernels of the form $k(x,x')=\kappa(r)$ where $r=x-x'$ and $\kappa$ is the Fourier transform of a Student's t-density,
\begin{equation}\label{eq:matern-density}
    \lambda(\omega) = C_{\ell,d}\left(1 + (\ell\|\omega\|_2)^2 \right)^{-\nu-d/2},
\end{equation}
for $\omega$ the $d$-dimensional frequency vector and
\begin{equation*}
    C_{\ell,d} = \ell^d\frac{\Gamma(\nu + d/2)}{\pi^{d/2} \Gamma(\nu)}
\end{equation*}
for $\ell, \nu > 0$ parameters of the kernel.

\paragraph{The Improved GP-UCB algorithm} The IGP-UCB algorithm is an approach to the kernelised bandit problem based on the classic GP-UCB Bayesian optimisation algorithm. For $f$ in the RKHS of a known kernel, it improves on the regret bounds of GP-UCB by a polylogarithmic factor and significantly improves empirical performance. For readers familiar with stochastic linear bandits, the step from GP-UCB to IGP-UCB mirrors that from ConfidenceBall \citep{dani2008stochastic} to OFUL \citep{abbasi2011improved}. 

IGP-UCB assumes a known bound $B$ on the RKHS norm of the function $f$ and a known bound on the subGaussianity constant $\subgauss$ of $\epsilon_t$. Then, for a chosen $\delta \in (0,1)$ and $Z \subset \R^d$, IGP-UCB uses a Gaussian process regressor mean $\mean{Z}{t-1}{x}$ and standard deviation $\std{Z}{t-1}{x}$, along with a confidence width multiplier 
\begin{equation}
 \beta^Z_t(\delta) = B + \subgauss\sqrt{2\left(\info{Z}{t} + 1 + \log(1/\delta)  \right)},   
\end{equation}
to construct $1-\delta$ probability confidence intervals for the restriction of $f$ to $Z$.  \citet{chowdhury2017kernelized} choose to use $Z=\dom$, that is, consider the whole domain, and show that selecting observations that maximise this upper confidence bound leads to regret bounded as $\mcO(\info{\dom}{T}\sqrt{T})$ with probability $1-\delta$. 

Intuitively, the IGP-UCB algorithm needs strong regularity assumptions to guarantee sublinear regret because the information gain term, $\info{\dom}{t}$, contained within the confidence width parameter $\beta^{\dom}_{t}(\delta)$, grows too quickly with $t$ otherwise. Specifically, in the case of the Mat\'ern kernel, $\info{\dom}{t}$ can be lower and upper bounded as
\begin{equation}\label{eq:matern-info-bounds}
    \porder{T^{\frac{d(d+1)}{d(d+1) + 2\nu}}} \quad \text{and} \quad \Omega(T^{\frac{d}{d+2\nu}})
\end{equation}   
respectively, where the former is given in \citet{srinivas2010gaussian} and the latter can be deduced from combining the bounds in \citet{valko2013finite} and \citet{scarlett17lower}. Our main contribution can be understood as addressing this issue. We provide a construction which leads to confidence intervals that grow only polylogarithmically with $T$.

\section{Main contribution} 

We introduce \textit{Partitioned Improved GP-UCB} ($\pi$-GP-UCB), an algorithm that at each time step constructs a closed cover of the domain $\dom$ and selects points by taking a maximiser of the IGP-UCB upper confidence bound constructed independently on each cover element. Throughout the paper, we will make use of the following two constants,
\[ 
b = \frac{d+1}{d+2\nu} \quad \text{and} \quad q = \frac{d(d+1)}{d(d+2)+2\nu},
\]
which depend on the RKHS only. Also, for a hypercube $A \subset \dom$ we will use $\diam{A}$ denote its $\ell^\infty$-diameter, i.e. the length of any of its sides.

\paragraph{The $\pi$-GP-UCB algorithm} 
Choose a $1-\delta$ confidence level. Let $\mcA_1$ be any set of closed hypercubes overlapping at edges only, of cardinality at most $\mcO(T^q)$, that covers the domain $\dom$.\footnote{For example, $\mcA_1=\{\dom\}$.} At each time step $t$ select a query location $x_t$ and then construct a new cover $\mcA_t$ as follows:

\emph{Point selection}. Fit an independent Gaussian process with $\reg = 1 + 2/T$ on each cover element $A \in \mcA_t$, conditioned only on data within $A$, and select the next point to evaluate by maximising
\[
    \UCB_t(x) =\!\! \max_{A \in \mcA_t \colon x \in A} \mean{A}{t-1}{x} + \bet{A}{t} \std{A}{t-1}{x},
\]
where $\bet{A}{t} = \beta^A_t(\delta / \widetilde{N}_t)$, with $\widetilde{N}_t = 4(t+1)^{bd}$.

\emph{Splitting rule}. Split any element $A \in \mcA_{t-1}$ for which
$\diam{A}^{-1/b} < |\dat{A}{t}|+1$ along the middle of each side, resulting in $2^d$ new hypercubes. Let $\mcA_t$ be the set of the newly created hypercubes and the elements of $\mcA_{t-1}$ which were not split.

\paragraph{Properties of algorithm} The construction of the cover $\mcA_t$ ensures the following two properties hold:
\begin{lem}\label{lem:partition-info}
Let $A$ be a subset of $\dom$ and suppose there exists a $\tau$ such that $A\in \mcA_{\tau}$. Let $\tau'(A)=\max\{t \colon A \in \mcA_t \}$. Then, for some $C>0$,
$\info{A}{\tau'(A)} \leq C \log T \log \log T$.
\end{lem}
\begin{lem}\label{lem:size-of-cover}
Let $\mcA_t$ be the covering set at time $t$. Suppose $|\mcA_1| = \mcO(T^q)$. Then, for $T$ sufficiently large, 
\[
|\cup_{t\leq T}\mcA_t| \leq C_{d,\nu}T^{q},
\] 
where $C_{d,\nu} > 0$ depends on $d$ and $\nu$ only.
\end{lem}
That is, on all cover elements, information gain can be bounded polylogarithmically for all $t\geq 1$, and the cardinality of all the covering sets generated up to time $T$ is sublinear in $T$. We will show that the regret of $\pi$-GP-UCB after $T$ steps is bounded by
\begin{equation*}
    R_T = \mcO(\info{A}{T}\textstyle{\sqrt{T|\cup_{t\leq T}\mcA_t|}})=\porder{\textstyle{\sqrt{T|\cup_{t\leq T}\mcA_t|}}} = o(T).
\end{equation*}
The formal statement of this result is:
\begin{theorem}\label{thm:ucb-regret}
Let $\dom = [0,1]^d$, let $\mcH_k(\dom)$  be the RKHS of a Mat\'ern kernel $k$ with parameter $\nu>1$ such that $k\leq 1$. Suppose $f: \dom \to \R$ satisfies $\|f\|_{\mcH_k} \leq B$ for a known $B$ and we observe $y(x_t) = f(x_t)+\epsilon_t$, where $\epsilon_t$ is $\subgauss$ sub-Gaussian. Then for any fixed $\delta \in (0,1)$, with probability at least $1-\delta$ the regret incurred by $\pi$-GP-UCB is bounded by $\porder{T^{\frac{d(2d+3)+2\nu}{d(2d+4) + 4\nu}}}$.
\end{theorem}
The properties of the cover also allows us to improve upon existing upper bounds for the information gain associated with a Mat\'ern kernel. To do this, we bound the information gain as the sum of the information gain on each cover element $A\in\mcA_t$, and therefore
\begin{equation*}
    \info{\dom}{T} = \mcO(\info{A}{T}|\cup_{t\leq T}\mcA_t|) = \porder{|\cup_{t\leq T}\mcA_t|},
\end{equation*}
which translates into the following bound:
\begin{theorem}\label{thm:gp-ucb-improved}
Information gain associated with the Mat\'ern kernel with parameter $\nu > 1$ after $T$ steps can be bounded as $\info{\dom}{T} = \porder{T^{\frac{d(d+1)}{d(d+2) + 2\nu}}}$.
\end{theorem}
This is a strict improvement over the bound presented in \citet{srinivas2010gaussian}, given here in \cref{eq:matern-info-bounds}. 

\paragraph{Practical considerations} While the main purpose of the cover construction within $\pi$-GP-UCB is to provide strong theoretical guarantees on regret, it also results in a significantly more scalable algorithm.

To see this, first note that fitting a Gaussian process requires updating the inverse of a kernel matrix $K_t$. Our cover construction means that inverses are only performed on a kernel matrix of a subset of the data, reducing both compute and memory costs. Empirically, this allows $\pi$-GP-UCB to scale much more favourably than GP-UCB. However, the construction offers no asymptotic improvement on computational complexity: once the algorithm comes close to convergence (enters the asymptotic regime), points are expected to land within small neighbourhoods of the optima, and therefore likely within the same cover element. 

Moreover, running Gaussian process UCB algorithms requires maximising the UCB index across the domain. Whereas the classic GP-UCB algorithms require the maximiser to be recomputed across all of the domain, under $\pi$-GP-UCB, the maximisation procedure need only be carried out over the cover elements containing the most recent observation and any newly created cover elements.

\section{Proof of results}
\input{proofs.tex}

\section{Empirical validation}
\input{experiments.tex}

\section{Related work}
While our theoretical guarantees are much stronger than those of existing practical kernelised methods like GP-UCB, IGP-UCB and KernelUCB, two existing methods achieve similar or better guarantees:

\textit{SupKernelUCB}. Introduced in \citet{valko2013finite}, SupKernelUCB uses a phased elimination procedure to create batches of observations that are independent of previous observations. This allows for the use of the stronger concentration inequalities that apply to i.i.d.~sequences, yielding a $\porder{(\deffe{\dom} T)^{1/2}}$ regret bound. While the algorithm is introduced for the case of a finite-armed bandit it can be extended to a continuum-armed bandit via a discretisation argument. SupKernelUCB is the kernelised version of stochastic linear bandit algorithm SupLinUCB \citep{auer2002using,chu2011contextual}. However, much like SupLinUCB \citep[Remarks 22.2]{banditbook}, it fails to achieve empirically sublinear regret even on very simple problems \citep{calandriello2019gaussian}. 

\textit{Hierarchical optimisation}. Extensions and generalisations of classic Lipschitz-continuity based methods enjoy strong regret guarantees under assumptions that are broadly similar to those in our work \citep{jones1993lipschitzian,munos2011optimistic,bubeck2011x}. Current upper bounds suggests for hierarchical methods are better for problems with a low degree of smoothness than the kernelised counterparts. It is an open question whether this holds in general, or whether the analysis of kernelised methods can be further improved.

\section{Discussion}\label{sec:discussion}
We have presented an algorithm for optimising functions in the RKHS of a  Mat\'ern family kernel, with a sublinear bound on regret for all smoothness parameters $\nu>1$ and demonstrated the practical effectiveness and scalability of the proposed algorithm. The empirical performance of $\pi$-GP-UCB might be improved by using the actual information gain, as opposed to an upper bound on the information gain, in determining when to split a set.
\bibliography{main}

\appendix
\onecolumn
\input{appendix.tex}
\end{document}

%% file: notation.tex
\definecolor{babypink}{rgb}{0.96, 0.76, 0.76}
\definecolor{babyblue}{rgb}{0.54, 0.81, 0.94}

\newcommand{\dat}[2]{D^{#1}_{#2}}
\newcommand{\xdat}[2]{X^{#1}_{#2}}

\newcommand{\info}[2]{\gamma^{#1}_{#2}}
\newcommand{\infomax}[2]{\hat{\gamma}(#1, #2)}
\newcommand{\mean}[3]{\mu^{#1}_{#2}(#3)}
\newcommand{\std}[3]{\sigma^{#1}_{#2}(#3)}
\newcommand{\var}[3]{(\std{#1}{#2}{#3})^2}
\newcommand{\bet}[2]{\widehat{\beta}^{#1}_{#2}}

\newcommand{\kmat}[2]{K^{#1}_{#2}}

\newcommand{\deffe}[1]{\widetilde{d}^{#1}}
\newcommand{\maxpts}[1]{N_{#1}}

\newcommand{\reg}{\alpha}
\newcommand{\porder}[1]{\widetilde{\mcO}(#1)}

\newcommand{\diam}[1]{\rho_{#1}}
\newcommand{\radius}[1]{\rho_{#1}}

\newcommand{\E}{\mathbb{E}}

\newcommand{\dom}{\mcX}

\newcommand{\1}[1]{1{\{#1\}}}

\newcommand{\UCB}{\text{UCB}}

\newcommand{\GP}{\mathcal{GP}}

\newcommand{\argmax}[1]{\operatornamewithlimits{argmax}_{#1}}

\newcommand{\subgauss}{L}

\newcommand{\N}{\mathbb{N}}
\newcommand{\R}{\mathbb{R}}

\newcommand{\mcA}{\mathcal{A}}
\newcommand{\mcB}{\mathcal{B}}

\newcommand{\mcF}{\mathcal{F}}

\newcommand{\mcH}{\mathcal{H}}

\newcommand{\mcM}{\mathcal{M}}

\newcommand{\mcO}{\mathcal{O}}

\newcommand{\mcV}{\mathcal{V}}

\newcommand{\mcX}{\mathcal{X}}

\newcommand{\mcZ}{\mathcal{Z}}

%% file: proofs.tex
    In order to prove \cref{thm:ucb-regret}, we will use the following concentration inequality:

\begin{lem}\label{lem:cover-concentration}
Given $\delta \in (0,1),$ for all $t \leq T$, for all $A \in \bigcup_{t\leq T} \mcA_t \eqqcolon \widetilde{\mcA}_T$, for all $x \in A$, we have 
\begin{align*}
    |\mean{A}{t-1}{x} - f(x)| \leq \bet{A}{t}\std{A}{t-1}{x},
\end{align*}
with probability $1-\delta$.
\end{lem}
To prove \cref{lem:cover-concentration}, we will use the following result, proven in \cref{sec:proof-of-image-bound}, which bounds the number of all cover elements that could ever be created within $t$ steps of running the algorithm by $\widetilde{N}_t= 4(t+1)^{bd}$. 
\begin{lem}\label{lem:image-bound}
There exists a set $B_t$ such that $|B_t| \leq \widetilde{N}_t$ and $\mcA_t \subset B_t$ with probability one.
\end{lem}
We present the full proof of \cref{lem:cover-concentration} in \cref{appendix:proof-concentration}. Here, we prove a weaker result that admits a much shorter proof.

By theorem 2 in \citet{chowdhury2017kernelized}, under the conditions of our \cref{thm:ucb-regret}, 
\[
    |\mean{A}{t-1}{x} - f(x)| \leq \beta^A_t(\delta)\std{A}{t-1}{x}
\]
with probability $1-\delta$ for any $A\subset\R^d$ compact. The weaker result then follows by taking a union bound over all $A \in B_T$, resulting in a confidence width multiplier of $\beta^A_t(\delta / \widetilde{N}_T)$. The full proof in the appendix allows us to use $\bet{A}{t} = \beta^A_t(\delta / \widetilde{N}_t)$ instead.

To prove \cref{thm:ucb-regret} we will also need the following bound relating the sum of predictive variances on a subset of the domain to information gain:

\begin{lem}\label{lem:variace-bound-infogain}
For any $\tau \geq 1$, sequence $(x_t:  x_t \in \dom, 1 \leq t \leq \tau)$, set $A\subset\dom$ and a Gaussian process estimator with $\alpha > 1$,
$\sum_{t=1}^{\tau} \1{x_t \in A}\var{A}{t-1}{x_t}\leq 4 \reg \info{A}{\tau}$.
\end{lem}

\begin{proof}
Since $0\leq \var{A}{t}{x},\alpha^{-1} \leq 1$ for all $x\in\dom$, we have that
$\reg^{-1} \var{A}{t}{x} \leq 2\log (1+\reg^{-1}\var{A}{t}{x})$, because for any $0\leq \sigma \leq 1$, $\sigma \leq 2\log(1+\sigma)$. Summing over $t$, we have
\begin{align*}
    \sum_{t=1}^{\tau} &\1{x_t \in A}\var{A}{t-1}{x_t} 
    \\&\leq \sum_{\substack{ t=1,\\ x_t \in A}}^{\tau} 2\alpha \log \left(1+\reg^{-1}\var{A}{t-1}{x_t}\right) = 4 \reg \info{A}{\tau},
\end{align*}
where the final equality is by Lemma 5.3 in \citet{srinivas2010gaussian}.
\end{proof}

\begin{proof}[Proof of \cref{thm:ucb-regret}]
By Cauchy-Schwarz, $R_T \leq (T\sum_{t=1}^T r^2_t)^{1/2}$. It therefore suffices to bound $\sum_{t=1}^T r^2_t$.

For each $x \in \dom$ let $A_t(x)$ be an element of $\mcA_t$ such that
\begin{equation*}
    A_t(x) \in \argmax{A \in \mcA_t : x \in A} \mean{A}{t-1}{x} + \bet{A}{t}\std{A}{t-1}{x}.
\end{equation*} 
That is, $A_t(x)$ is an element of $\mcA_t$ on which the upper confidence bound associated with $x$ is the highest.

For any $x^\star \in \argmax{x\in\dom} f(x)$, we have that $\UCB_t(x_t) \geq \UCB_t(x^\star)$ due to the manner in which points $x_t$ are selected. Expanding this expression and applying \cref{lem:cover-concentration}, we bound $r_t$ as
\begin{equation*}
    r_t = f(x^\star) - f(x_t) \leq 2\bet{A_t(x_t)}{t} \std{A_t(x_t)}{t-1}{x_t}
\end{equation*}
with probability $1-\delta$ for all $t \leq T$. 

Denote $\widetilde{\mcA}_T = \cup_{t\leq T} \mcA_t$, the set of all cover elements created until time $T$, and define the initial time for an element $A \in\widetilde{\mcA}_T $, as $\tau(A) = \min \{t \colon A \in \mcA_t \}$ and the terminal time as $\tau'(A) = \max \{t \colon A \in \mcA_t \}$. We have
\begin{align*}
\sum_{t=1}^T r_t^2 
    &\leq 4\sum_{t=1}^T(\bet{A_t(x_t)}{t})^2 \var{A_t(x_t)}{t-1}{x_t}  \\
    &\leq 4 \sum_{t=1}^T \sum_{A\in \mcA_t} \1{x_t \in A} (\bet{A}{t})^2 \var{A}{t-1}{x_t} \\
    &= 4 \sum_{A \in \widetilde{\mcA}_T} \sum_{t=\tau(A)}^{\tau'(A)} \1{x_t \in A}  (\bet{A}{t})^2  \var{A}{t-1}{x_t} \\
    &\leq 48  \sum_{A \in \widetilde{\mcA}_T}(\bet{A}{\tau'(A)})^2 \info{A}{\tau'(A)} \tag{$\ast$} 
\end{align*}
The final inequality uses monotonicity of $\bet{A}{t}$, \cref{lem:variace-bound-infogain} and $\alpha<3$. By \cref{lem:size-of-cover}, the number of summands in $(\ast)$ is $\mcO(T^q)$. As $(\bet{A}{{\tau'(A)}})^2=\mcO(\info{A}{{\tau'(A)}}+\log(T))$ for all $A \in \widetilde{\mcA}_T$, \cref{lem:partition-info} completes the proof. 
\end{proof}

\begin{proof}[Proof of \cref{thm:gp-ucb-improved}]
 From proposition 5 in \citet{calandriello2019gaussian},
\begin{align*}
    \deffe{A} \leq \log | I+\reg^{-1} \kmat{A}{T} | 
    &\leq \deffe{A}(1+\log(\reg^{-1} \|\kmat{A}{T}\|_2 +1)).
\end{align*}
Noting that $\|\kmat{A}{T}\|_2 = \lambda^A_1 \leq \sum_{t=1}^{|X_T^A|} \lambda^A_t = \mathrm{Tr} (\kmat{A}{T}) = |X_T^A| \leq T$, and taking $A = \dom$ we have,
\[
\deffe{\dom} \leq 2\info{\dom}{T}\leq \deffe{\dom}(1+\log(\reg^{-1}T +1)).
\]
We now proceed to bound $\deffe{\mcX}$. At each time step $t$ choose $\mcA_t$ to be as in the closed cover of $\dom$ from $\pi$-GP-UCB. Let $\widetilde{\mcA}_T = \cup_{t \leq T} \mcA_t$, as before. 

For any $A \in \widetilde{\mcA}_T$, let $\tau'(A) = \max\{t\colon A\in \mcA_t \}$. Then, because $x_t$ is in at least one of the $A_t$ for all $t$ and by monotonicity of predictive variance, we have
\begin{align*}
    \deffe{\mcX} \!\leq\! \frac{1}{\reg}\sum_{t=1}^T\sum_{A \in \mcA_t}\! \1{x_t \in A} \var{A}{\tau'(A)-1}{x_t}.
\end{align*}
Interchanging the order of summation and applying \cref{lem:variace-bound-infogain}, we have that this is upper bounded by $4 \!\sum_{A \in \widetilde{\mcA}_T}\info{A}{\tau'(A)}$. Then, by \cref{lem:partition-info}, we have that 
\begin{equation*}
    \deffe{\mcX} \leq 4 C|\widetilde{\mcA}_T| \log T \log\log T.
\end{equation*}
By \cref{lem:size-of-cover}, $|\widetilde{\mcA}_T| = \mcO(T^q)$. Therefore 
\[
\info{\mcX}{T} \leq C^\prime T^q \log T (\log \log T) (1 + \log(\reg^{-1}T+1)),
\]
for some $C^\prime$ independent of $T$.
\end{proof}

\subsection{Proof of \cref{lem:partition-info}, information gain on a~cover element}\label{sec:infogain}

To prove \cref{lem:partition-info}, we rely on the following theorem from \citet{srinivas2010gaussian}. It provides a bound on the maximum information gain after $N$ samples, $\infomax{A}{N}$, in terms of the operator spectrum of the kernel $k$ with respect to a uniform covariate distribution.\footnote{The original statement of \Cref{thm:theorem8} in \citet{srinivas2010gaussian} assumed strong conditions on the sample paths of the Gaussian process associated with the kernel. These assumptions were not necessary for the theorem itself, and were present due to the later use of the bounds for optimising Gaussian process samples.}
\begin{theorem}{\citep[Theorem 8]{srinivas2010gaussian}}\label{thm:theorem8}
Suppose that $A \subset \mathbb{R}^d$ is compact, and $k$ is kernel continuously differentiable in a neighbourhood of $A$. Let $\mathcal{S}(s_0) = \sum_{s>s_0} \lambda_s$ where $(\lambda_s)_s$ is the operator spectrum of $k$ with respect to the uniform distribution over $A$. Pick $\zeta > 0$ and let $n_N = (4\zeta + 2)\mcV_A N^\zeta \log N$, where $\mcV_A$ is the volume of $A$. Then $\infomax{A}{N}$ is bounded by
\begin{align}
     C\!\!\! &\max_{r=1,\dots,N}\! \Big[
    s_0 \log\frac{rn_N}{\reg}
    + \left(4\zeta + 2\right)\mcV_A\log N \reg^{-1} \nonumber\\
    &\ \times\left(1-\frac{r}{N}\right)\left(N^{\zeta+1} \mathcal{S}(s_0) + 1\right)\Big] + \mcO\!\left(N^{1-\zeta/d}\right)\!, \label{eq:info-max}
\end{align}
where $C = \frac{1}{2}/(1-\frac{1}{e})$ for any $s_0 \in \N \cap [1,n_N]$.
\end{theorem}
The operator spectrum of the Mat\'ern kernel, required to use \cref{thm:theorem8}, can be bounded using the following.
\begin{theorem}{\citep[Theorem 2]{seeger2008information}}\label{thm:seeger}
Let $K(r)$ be an isotropic covariance function on $\mathbb{R}^d$ satisfying the conditions of Widom's theorem \citep{widom1963asymptotic}, with a spectral density $\lambda(\cdot)$. Suppose that the covariate distribution $\mu$ has bounded support and a bounded density, in that $\mu(x) \leq D$ for all $x$ and $\mu(x) = 0$ for all $\|x\|_2 > R$. Then,
\begin{equation*}
    \lambda_s \leq D(2\pi)^d \lambda \left(C_dR^{-1} s^{1/d}\right)(1 + o(1))
\end{equation*}
asymptotically as $s\to\infty$, where $C_d>0$.
\end{theorem} 
The required spectral density of a Mat\'ern kernel is given in \cref{eq:matern-density}. By \citet{seeger2008information}, it satisfies the conditions of Widom's theorem.

\begin{proof}[Proof of \cref{lem:partition-info}]
As the information gain is a sum of non-negative elements, we have for any $M\leq N$,
\[
\infomax{A}{M} \leq \infomax{A}{N}.
\] 
It therefore suffices to bound the maximum number of points that can fall in a partition $A$ before it would split. Let $\maxpts{A}$ denote this quantity. The proof proceeds in two parts: first we bound $\infomax{A}{\maxpts{A}}$ as a function of $\maxpts{A}$, then we bound $\maxpts{A}$ in terms of the horizon.

From \cref{eq:matern-density}, we have that the spectral density for the Mat\'ern kernel satisfies 
\begin{equation*}
    \lambda(\omega)=C_{\ell,d}(1+(\ell\omega)^2)^{-\nu-d/2} \leq C'_{\ell,d} \omega^{-(2v + d)}.
\end{equation*} where $C'_{\ell,d}=C_{\ell,d}\ell^{-(2v + d)}$. Utilising this within \cref{thm:seeger} with a uniform covariate distribution,\footnotemark
\begin{equation*}
    \lambda_s \leq C \radius{A}^{-d}  (s^{1/d} \radius{A}^{-1})^{-(2v+d)}  = C \radius{A}^{2\nu} s^{-(2\nu+d)/d},
\end{equation*}
for some constant $C>0$, where we used $\mu(x) = \mcV_A^{-1} = \diam{A}^{-d}$ for all $x \in A$, as $A$ is a $d$-dimensional cube.

As the bound on $\lambda_s$ for large $s$ is monotonically decreasing, we can bound the tail of the Mat\'ern kernel operator spectrum $\mathcal{S}(s_0)$ as
\begin{align*}
    \sum_{s > s_0} \lambda_s 
    &\leq C\radius{A}^{2\nu} \int_{s=s_0}^\infty  s^{-(2\nu+d)/d} =\mcO( \radius{A}^{2\nu}).
\end{align*}

\footnotetext{Errata: we forgot the $o(1)$ term appearing in the bound on $\lambda_s$ in \cref{thm:seeger}. We address this omission in chapter 4 of \citet{janz2022sequential}---the overall regret bound given in \cref{eq:final-regret} remains unchanged.}

We now apply \cref{thm:theorem8} in order to bound $\infomax{A}{\maxpts{A}}$. Choose $\zeta=d$, then \[\maxpts{A}^{\frac{d(2\nu-1)}{2\nu+d}}<n_{\maxpts{A}}<\maxpts{A}^d\] for $\maxpts{A}$ sufficiently large. Choose $s_0 = \lfloor \log \log N\rfloor$.

For these parameter choices and for any $r \leq \maxpts{A}$,
\[
s_0\log \frac{r n_{{\maxpts{A}}}}{\alpha} = \mcO(\log \maxpts{A}\log \log \maxpts{A}).
\]
As $1-r/\maxpts{A}\leq 1$, the second term in the maximum in \cref{eq:info-max} is $\mcO(\mcV_A\maxpts{A}^{d+1}\diam{A}^{2\nu}\log\maxpts{A})$. As the diameter satisfies $\maxpts{A}\leq \diam{A}^{-1/b} < \maxpts{A}+1$ by the definition of $\maxpts{A}$ this term is $\mcO(\log \maxpts{A})$. The final term in \cref{eq:info-max} is~$\mcO(1)$. 

All that remains is to bound $\maxpts{A}$. We consider two cases: first, if $A \in \mcA_1$ then $\maxpts{A} \leq \diam{A}^{-1/b} =\mcO(T^{\frac{q}{bd}})$. Otherwise, $A$ was created by some set $A'$ splitting, for which $N_{A'}<T$. Then,
\begin{align}
    N_A &\leq \diam{A}^{-1/b} = 2^{-1/b}\diam{A'}^{-1/b} \nonumber \\ &\leq 2^{-1/b}(\maxpts{A'}+1) <2^{-1/b}(T+1). \nonumber 
\end{align}
In both cases, $N_A=\mcO(T)$. \qedhere

\end{proof}

\subsection{Proof of \cref{lem:size-of-cover}, bound on size of cover }\label{sec:size-of-cover}

\newcommand{\splitset}{\Theta_T}

\begin{figure*}[h]
    \centering
    \includegraphics{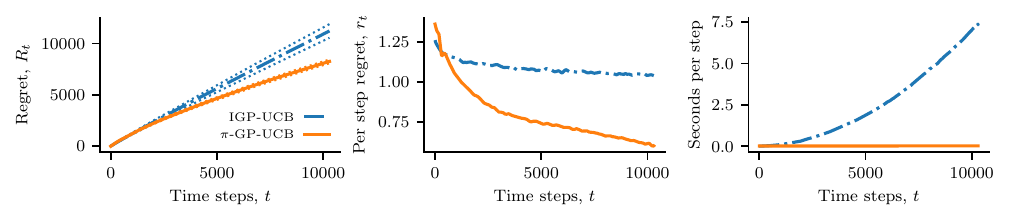}\vspace{-8mm}
    \caption{Comparison of performance of $\pi$-GP-UCB and IGP-UCB on the synthetic benchmark. Left, regret, $95\%$ confidence interval dotted. Centre, per step regret, smoothed using a convolution with a top hat of length $200$. Right, wall-clock time per step, in seconds.}
    \label{fig:results}
\end{figure*}

\begin{proof}
First, since any element $A \in \widetilde{\mcA}_T$ was either in $\mcA_1$ or was created by the splitting of a cover element $A' \in \widetilde{\mcA}_T \setminus \mcA_T$ into $2^d$ elements, we have that
\[ 
|\widetilde{\mcA}_T| = |\mcA_1| + 2^d|\widetilde{\mcA}_T \setminus \mcA_T|.
\]
By assumption $|\mcA_1| = \mcO(T^q)$. Denote $\splitset = \widetilde{\mcA}_T \setminus \mcA_T$. We now upper bound $|\splitset|$. Take the diameter of any element $A \in \mcA_1$ to be $\rho_0 \in (0,1]$ (e.g.~$\rho_0 = 1$ if $\mcA_1 = \{\dom \}$). For any given $T$, we have that 
\[
|\Theta_T|\ \leq\! \max_{x_1, \dotsc, x_T} |\Theta_T|\ = \! \max_{x_1,\dotsc,x_T}\sum_{i=0}^\infty \sum_{A \in \Theta_T} \!\1{\diam{A}=2^{-i}\rho_0}.
\] 
We upper bound the solution to this maximisation problem by considering just a subset of the constraints imposed by the splitting procedure.

\newcommand{\mintime}{\tau(A)}
\newcommand{\maxtime}{\tau'(A)}

First constraint: we have a \emph{budget constraint} derived from placing $T$ points. Let $\tau(A) = \min \{ t \colon A \in \mcA_t \}$ and $\tau'(A) = \max \{t \colon A \in \mcA_t \}$, and suppose there exists an $M(\diam{A}) \leq |\xdat{A}{\maxtime}| - |\xdat{A}{\mintime}|$ for all $A \in \splitset$. Then
\begin{align*}
    \sum_{i=0}^\infty &M(2^{-i} \rho_0) \sum_{A \in \Theta_T} \1{\diam{A} = 2^{-i} \rho_0} \\
    &\leq \sum_{i=0}^\infty \sum_{A \in \Theta_T} (|\xdat{A}{\maxtime}| - |\xdat{A}{\mintime}|) \1{\diam{A} = 2^{-i}\rho_0} \\
    & = \sum_{A \in \Theta_T}^{\phantom{T}} |\xdat{A}{\maxtime}| - |\xdat{A}{\mintime}|
    \ \leq \!\sum_{A \in \widetilde{\mcA}_T}^{\phantom{T}} |\xdat{A}{\maxtime}| - |\xdat{A}{\mintime}| \\
    & = \sum_{A \in \widetilde{\mcA}_T} \sum_{t=\mintime}^{\maxtime} \1{x_t \in A}
    = \sum_{t=1}^T \sum_{A \in \mcA_T} \1{x_t \in A} \\
    & = \sum_{t=1}^T |\{ A \in \mcA_t \colon x_t \in A \}| \ \leq 2^dT.
\end{align*}
Now we find a suitable $M(\cdot)$, which we shall refer to as the \emph{cost} of splitting an element $A \in \Theta_T$. Because $A$ split,
\[
|\xdat{A}{\maxtime}| + 1 > \diam{A}^{-1/b} \geq |\xdat{A}{\maxtime}|
\]
Suppose that $A'$ is the element that split to create $A$. Then $\diam{A'} = 2\diam{A}$ and $\tau'(A') + 1 = \mintime$. Therefore
\begin{align*}
|\xdat{A}{\maxtime}| - |\xdat{A}{\mintime}| 
&\geq |\xdat{A}{\maxtime}| - |\xdat{A'}{\tau'(A')}| - 1 \\
&\geq \diam{A}^{-1/b}(1-2^{-1/b}) - 2=M(\diam{A}).
\end{align*}

Second constraint: a \emph{supply constraint}. There are at most $\lceil \rho_0^{-d} \rceil$ elements of diameter $\rho_0$, and therefore at most $\lceil \rho_0^{-d} \rceil 2^{di}$ elements of diameter $2^{-i}\rho_0$ can be split, leading to  
\[
\sum_{A \in \Theta_T}^{\phantom{T}} \1{\diam{A} = 2^{-i}} \leq \left\lceil\rho_0^{-d} \right\rceil 2^{di}.
\]

Since $M(2^{-i}\rho_0)$ increases with $i$, the solution to the relaxed optimisation problem will be to buy all the available $A$ with smallest diameter, subject to the supply and budget constraints. Suppose the smallest $A$ split with this strategy has a diameter $2^{-z}\rho_0$ for some $z \in \N$. Then, since the supply constraint is binding and budget constraint is satisfied, we have that 
\[\sum_{i=0}^{z-1} M(2^{-i}\rho_0) 2^{di}\left\lceil\rho_0^{-d} \right\rceil \leq 2^dT.\]
Writing $\rho_0 = CT^{\alpha}$ for some $C>0$ and $\alpha\leq q/d$ and using the geometric series formula to solve for $2^z$, we obtain $2^z = \mcO(T^{\frac{b}{bd+1}})$, a quantity independent of $\alpha$. Counting all the cover elements of diameters $2^0\rho_0, \dotsc, 2^z\rho_0$, we have that $\sum_{i = 0}^{z} 2^{di} = \mcO(2^{dz})$, which is $\mcO(T^q)$. Since this was a construction that maximises $\Theta_T$, we have that $\widetilde{\mcA}_T = \mcO(T^q)$. 
\end{proof}

\subsection{Proof of \Cref{lem:image-bound}}\label{sec:proof-of-image-bound}

\newcommand{\numsplits}{g_T}
\newcommand{\numsplitslevel}{g_T^i}
\begin{figure*}[h]
    \centering
    \includegraphics{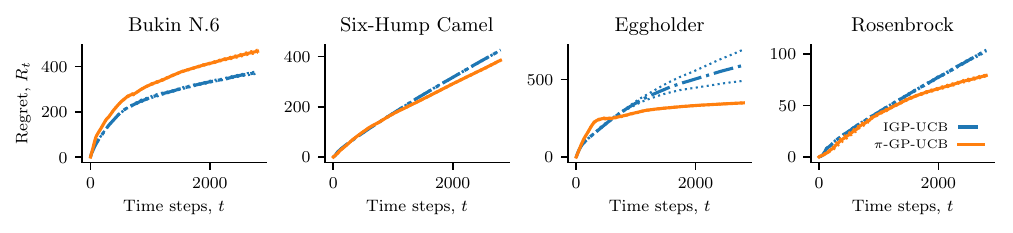}\vspace{-8mm}
    \caption{Comparison of performance of $\pi$-GP-UCB and IGP-UCB on the common two-dimensional global optimisation benchmark problems. Regret and $95\%$ confidence intervals plotted.}
    \label{fig:practical-results}
\end{figure*}

\begin{proof}
Let $B^0 = \mcA_1$ and define recursively $B^{i+1}$ to be the set of the hypercubes created by splitting each element in $B^i$ into $2^d$ hypercubes. Let $\rho_0 \in (0,1]$ be the diameter of elements in $\mcA_1$. We have that
\[
\mcA_t \subset \bigcup_{i \geq 0} B^{i} \implies \mcA_t \subset \bigcup_{i \geq 0}\! \left(B^{i} \cap \mcA_t \right).
\]
Suppose there exists $Z \in B^{i} \cap \mcA_t$ for some $i\geq 0$. By the splitting condition,
\[
Z \in \mcA_t\setminus B^0 \implies \diam{Z} > (|\xdat{Z}{t}| + 1)^{-b} \geq (t+1)^{-b}.
\]
Also, $Z \in B^{i}$ implies $\diam{Z} \leq 2^{-i}\rho_0$.
Therefore, $(t+1)^{-b} \leq \diam{Z} \leq 2^{-i}\rho_0$ so $2^i \leq (t+1)^b\rho_0=2^{J_t}.$ Let $B_t = \bigcup_{i \leq J_t} B^{i}$. We have $B^{i} \cap \mcA_t = \emptyset$ for all $ i > J_t$ and hence $\mcA_t \subset B_t$. We now bound the cardinality of $B_t$. We have $|B^{i}| = \lceil\rho_0^{-d}\rceil2^{di}$, so
\begin{align*}
|B_t|&=\sum_{i=0}^{J_t} |B^i| = \lceil\rho_0^{-d}\rceil\sum_{i=0}^{\lfloor J_t\rfloor} 2^{di} 
\leq  \lceil\rho_0^{-d}\rceil2^{dJ_t+1} \\
&\leq 4(t+1)^{db}. \qedhere
\end{align*}
\end{proof}

%% file: experiments.tex
We present an empirical comparison of $\pi$-GP-UCB and IGP-UCB on two types of functions: first, synthetic functions in the Mat\'ern kernel RKHS, where the conditions of the theory for both algorithms are met; second, standard global optimisation baselines, corresponding to a more realistic setting where the RKHS norm of the target function is not known.

We run both IGP-UCB and $\pi$-GP-UCB using a Mat\'ern kernel with parameters $\nu=3/2$ and $\ell=1/5$, and use a regularisation parameter $\reg = 1$. For simplicity, the problems are discretised onto a regular grid, such that $\dom = \{x_1, \dotsc, x_n \colon n=30^d\}$. We use a confidence parameter $\delta=1/10$, and compute the quantities $\info{\dom}{t}$ and $\info{A}{t}$ exactly at each time step. For our method, $\pi$-GP-UCB, we use an initial partition $\mcA_1$ of cardinality approximately $T^q$.

\paragraph{Synthetic functions} We benchmark on a set of synthetic functions which satisfy the assumptions behind the regret bounds of both algorithms. We construct each function $f$ by sampling $m=30d$ points, $\hat{x}_1, \dotsc, \hat{x}_m$, uniformly on $[0, 1]^d$, and $\hat{a}_1, \dotsc, \hat{a}_m$ each independent uniform on $[-1, 1]$ and defining $f(x) = \sum_{i=1}^m \hat{a}_j k(\hat{x}_j, x)$ for all $x \in \mcX$, where $k$ is a Mat\'ern $\nu=3/2$ kernel with lengthscale $\ell=1/5$. Both algorithms are given access to the exact RKHS norm of this function, computed as $\|f\|^2_k = \sum_{i=1}^m\sum_{j=1}^m \hat{a}_i\hat{a}_j k(\hat{x}_j, \hat{x}_i)$. Evaluations of these functions are corrupted with independent additive noise $\epsilon_t$ sampled uniformly on $[-1, 1]$, and both algorithms use $\subgauss = 1$, the corresponding subGaussianity constant. 

We present regret and run-times on this set of synthetic functions for $d=1,2,3$ in \cref{tab:full-results} and plot the results for $d=2$ in \cref{fig:results}. The results suggest that $\pi$-GP-UCB not only provides an improved worst-case analysis, but may also yield improved empirical performance. Moreover, we see that $\pi$-GP-UCB is significantly more scalable. We note that the runtime of $\pi$-GP-UCB is high in the case $d=1$. This is because in one dimension, $\pi$-GP-UCB converges rapidly, losing its computational advantages thereafter.
\begin{table}
    \centering
    \begin{tabular}{c c c c c}
    \toprule
    \multirow{2}{*}{$d$} & \multicolumn{2}{c}{IGP-UCB} & \multicolumn{2}{c}{$\pi$-GP-UCB} \\
    \cmidrule(lr){2-3}
    \cmidrule(lr){4-5}
    & regret & runtime & regret & runtime \\
    \midrule
    1 & 0.11 & 82 min & 0.09 & 12 min \\
    2 & 0.71 & 118 min & 0.52 & 97 s \\
    3 & 0.97 & 30 hours & 0.77 & 29 min \\
    \bottomrule
    \end{tabular}
    \caption{Tabulation of results for synthetic task with $d=1,2,3$. Averages values over 12 runs for a horizon $T=10000$. Regret is expressed as fraction of the regret incurred by pulling the arms uniformly.}
    \label{tab:full-results}
\end{table}

\paragraph{Global optimisation baselines} We also provide a comparison of the empirical performance of $\pi$-GP-UCB and IGP-UCB on four common two-dimensional global optimisation baselines,\footnotemark~with results plotted in \cref{fig:practical-results}. The function outputs are scaled to $[-1,1]$, and corrupted with additive noise uniform on $[-0.1, 0.1]$. Whilst we provide the tight subGaussianity parameter to both algorithms, we run both with the RKHS norm parameter $B=1$, corresponding to a realistic setting where the RKHS norm is not known in advance. Our algorithm, $\pi$-GP-UCB, performs competitively across this set of problems.
\footnotetext{See \url{http://www.sfu.ca/~ssurjano} \citep{surjanovic2013virtual} for information on these functions.}

%% file: appendix.tex
\section{Proving \cref{lem:cover-concentration}, concentration result}\label{appendix:proof-concentration}

\begin{lem}\label{lem:martingale-argument}
Let $(\Omega, \mcF, (\mcF_t)_{t\geq 0}, \Pr)$ be a filtered space, with $\mcF_0$ the trivial sigma algebra. Let $(x_t)_{t\geq 1}$ be a previsible sequence $x_t \colon \Omega \mapsto \R^d$ and let $(\epsilon_t)_{t \geq 1}$ with $\epsilon_t \colon \Omega \mapsto \R$ be a sequence of random variables adapted to the filtration, with $\epsilon_t$ 1-subGaussian conditionally on $\mcF_{t-1}$ for all $t$. Let $(N_t)_{t\geq 1}$ be a non-decreasing sequence of integers. Let $(\mcA_t)_{t\geq 0}$ be a sequence of random sets $\mcA_t \colon \Omega \mapsto 2^{2^\dom}$, such that $\mcA_0$ is $\mcF_0$-measurable, $(\mcA_t)_{t\geq 1}$ previsible and $1 \leq |\mcA_t| < N_{t}$ almost surely for all $t \geq 0$. Let $k\colon \dom \times \dom \mapsto \R$ be a symmetric, positive-semidefinite kernel. Then for any given $\delta \in (0,1)$ and $\eta > 0$, for all $t \geq 0$ and all $A\in\mcA_t$ we have
\begin{equation*}
    \|\epsilon^A_{1:t}\|_{(I + (K^A_t+\eta I)^{-1})^{-1}} \leq 2 \log \left(\det (K^A_t + I + \eta I)^{\frac{1}{2}} N_t / \delta \right),
\end{equation*}
with probability $1-\delta$, where $\epsilon_{1:t}^A$ for the random vector that is the concatenation of $(\epsilon_z \colon x_z \in A)_{z=1}^t$.
\end{lem}

\begin{proof}
For a function $g\colon \dom \mapsto \R$ and a sequence of real numbers $(a_t)_{t\geq 1}$, define
\[
\Delta^{g,n}_t = \exp\left\{(g(x_t) + a_t)\epsilon_t - \tfrac{1}{2}(g(x_t) + a_t)^2\right\},
\]
with $\Delta^{g,n}_0$ defined as equal to $1$ almost surely. Then $\Delta^{g,n}_t$ is $\mcF_t$ measurable for all $t\geq 0$. By the conditional subGaussianity of $\epsilon_t$, we have that $\E[\Delta^{g,n}_t | \mcF_{t-1}] \leq 1$ for all $t \geq 0$ almost surely. For a set $A\in 2^\dom$, define 
\begin{equation*}
    \mcM_t^{g,n}(A) = \Delta_0^{g,n} \prod_{z=1}^t (\Delta_z^{g,n})^{\1{x_z \in A}}.
\end{equation*}
Then, for any $A\in 2^\dom$ and all $t\geq 1$, 
$\E[\mcM_t^{g,n}(A) | \mcF_{t-1} ] \leq \mcM_{t-1}^{g,n}(A)$ and $\E[\mcM_t^{g,n}(A)] \leq 1$.

Let $\zeta = (\zeta_t)_{t \geq 1}$ be a sequence of independent and identically distributed Gaussian random variables with mean $0$ and variance $\eta > 0$, independent of $\mcF_\infty = \bigcup_{t\geq 0} \mcF_t$. Let $h$ be a random real valued function on $A$ distributed according to the Gaussian process measure $\GP(0, k\big|_A)$, where $k\big|_A$ is the restriction of $k$ to $A$. Define 
\[   
M_t^A = \E[\mcM^{h,\zeta}_t(A) | \mcF_\infty].
\]
Then $M^A_t$ is itself a non-negative supermartingale bounded in expectation by 1. Define $\widetilde{M}_t^A = M_t^A / N_t$. Since $N_t \geq 1$ for all $t \geq 0$ and is non-decreasing, $\widetilde{M}_t^A$ is~a non-negative supermartingale bounded in expectation by $1/N_t$. 

For $A \in \mcB(\dom)$, let $B^A_t = \{\omega \colon \widetilde{M}_t^A > 1/\delta\}$ and $B_t = \bigcup_{A \in \mcA_t} B^A_t$. Define the stopping time $\tau(\omega) = \inf\{t \colon \omega \in B_t \}$. Then
\[
    \Pr[B_\tau^A| \mcF_{\tau-1}] \leq \delta \E[ \widetilde{M}_\tau^A | \mcF_{\tau-1}] = \delta \E[M_\tau^A | \mcF_{\tau-1}]/N_\tau \leq \delta/N_\tau M^A_{\tau-1} \quad \text{a.s.}
\]
We now examine the probability of $B_\tau$. We have
\begin{align*}
    \Pr[B_\tau] = \E\left[\Pr[B_\tau | F_{\tau-1}]\right]
    &\leq \sum_{A\in\mcB(\dom)} \E\left[ \1{A \in \mcA_\tau}\Pr[B_\tau^A| \mcF_{\tau-1}]\right] 
    \leq \delta/N_\tau \sum_{A\in\mcB(\dom)}\E\left[ \1{A \in \mcA_\tau}M^A_{\tau-1}\right].
\end{align*}
The final expectation is complicated by the fact that the event $\{A \in \mcA_t \}$ is not independent of $M_{t-1}^A$. However,
\[
    \{A \in \mcA_t \} \subset \{A \in \mcZ \colon \mcZ \subset \mcB(X), |\mcZ| \leq N_t \}.
\]
The latter event holds with probability $1$ for all $t \geq 1$, and is therefore independent of $M_t^A$. This gives,
\begin{align}\label{eq:bound-bad1}
    \Pr[B_\tau] 
    &\leq \delta/N_\tau \sum_{A\in\mcB(\dom)}\E\left[ \1{A \in \mcA_\tau}M^A_{\tau-1}\right] 
    \leq \delta/N_\tau \sum_{A\in\mcB(\dom)}\E\left[ \1{A \in \mcZ \colon |\mcZ| \leq N_t}M^A_{\tau-1}\right] \\\label{eq:bound-bad2}
    &= \delta/N_\tau \E[M^A_{\tau-1}] \sum_{A\in\mcB(\dom)}\E\left[ \1{A \in \mcZ \colon |\mcZ| \leq N_t}\right] 
    \leq \delta,
\end{align}
and consequently
\begin{equation}\label{eq:stopping-time-construction}
    \Pr\left[ \cup_{t \geq 0} B_t\right] = \Pr[\tau < \infty] = \Pr[B_\tau, \tau < \infty] \leq \Pr[B_\tau] \leq \delta.
\end{equation}

Finally, by comparing with the proof of Theorem 1 in \citet{chowdhury2017kernelized}, it can be verified that
\[
    M_t^A = \det (K^A_t + I + \eta I)^{-\frac{1}{2}} \exp \left\{ \tfrac{1}{2} \|\epsilon^A_{1:t}\|_{(I + (K^A_t+\eta I)^{-1})^{-1}} \right\}.
\]
The statement of the lemma follows from using this expression with \cref{eq:stopping-time-construction}, and noting that logarithms preserve order.
\end{proof}

\begin{proof}[Proof of \cref{lem:cover-concentration}] To prove \cref{lem:cover-concentration}, first since $|\mcA_t| \leq |\widetilde{\mcA}_t| \leq \widetilde{N}_t$, we can use $\widetilde{N}_t$ from \cref{lem:image-bound} as the bound $N_t$ required for \cref{lem:martingale-argument}. Then the proof of \cref{lem:cover-concentration} follows the proof of theorem 2 in \citet{chowdhury2017kernelized}, with our concentration inequality, \cref{lem:martingale-argument}, used instead of their theorem 1.
\end{proof}